\documentclass{article}
\usepackage{ijcai24}

\usepackage{times}
\usepackage{soul}
\usepackage{url}
\usepackage[hidelinks]{hyperref}
\usepackage[utf8]{inputenc}
\usepackage[small]{caption}
\usepackage{graphicx}
\usepackage{amsmath}
\usepackage{amsthm}
\usepackage{booktabs}
\usepackage{algorithmic}
\usepackage[switch]{lineno}

\usepackage[dvipsnames]{xcolor}
\usepackage{color,epsfig}
\usepackage{pifont}
\usepackage{latexsym}
\usepackage{amsmath,amsfonts,amssymb,nicefrac}
\usepackage{booktabs,tabularx,multirow,caption}
\usepackage{graphicx,tikz}
\usepackage{todonotes}
\usepackage{skull}
\usepackage{multicol}
\usepackage{verbatim}
\usepackage[ruled,linesnumbered, noline]{algorithm2e}

\def \opt{{\mathsf{opt}}}

\def\USM{\mathsf{UnSubMax}}

\def\SMK{\mathsf{SMK}}
\def\SMC{\mathsf{SMC}}
\def \opt{{\mathsf{OPT}}}
\def\F{{\mathcal F}}

\def\eE{{\mathbb E}}
\def\AST{\mathsf{AST}}

\def\PAR{\mathsf{ParSKP1}}
\def\PR{\mathsf{ParSKP2}}
\def\PRK{\mathsf{ParKnapsack}}
\def\RaAc{\mathsf{SmkRanAcc}}

\def\ThS{\mathsf{ThreshSeq}}

\def\Rand{\mathsf{RandBatch}}

\def\Get{\mathsf{GetSEQ}}
\def\RLA{\mathsf{RLA}}
\def\RM{\mathsf{RM}}
\def\IS{\mathsf{IS}}
\def\MWC{\mathsf{MWC}}

\def\NSMK{\mathsf{\mbox{non-monotone } SMK}}

\def\sO{\tilde{O}}
\def\E{{\mathbb E}}

\newtheorem{theorem}{Theorem}[section]
\newtheorem{lemma}{Lemma}

\newtheorem{definition}{Definition}

\title{Improved Parallel Algorithm for Non-Monotone Submodular Maximization under Knapsack Constraint}

\author{
	Tan D. Tran$^1$\and  	Canh V. Pham$^1$\footnote{Corresponding author}
	\and
	Dung T. K. Ha$^{2}$
	\and
	Phuong N.H. Pham$^{3}$
	\affiliations
	$^1$ORLab, Faculty of Computer Science, Phenikaa University, Hanoi, Vietnam
	\\	
	$^2$ Faculty of Information Technology, VNU University of Engineering and Technology, Hanoi, Vietnam
	\\
	$^3$ Faculty of Information Technology, Ho Chi Minh City University of Industry and Trade, Ho Chi Minh, Vietnam
	\emails
	canh.phamvan@phenikaa-uni.edu.vn,
	\{22027005, 20028008\}@vnu.edu.vn,
	phuongpnh@huit.edu.vn
}

\begin{document}	
	\maketitle
	
	\begin{abstract}
		
		This work proposes an efficient parallel algorithm for non-monotone submodular maximization under a knapsack constraint problem over the ground set of size $n$. Our algorithm improves the best approximation factor of the existing parallel one from $8+\epsilon$
		to $7+\epsilon$ with $O(\log n)$ adaptive complexity. 
		The key idea of our approach is to create a new alternate threshold algorithmic framework. This  strategy alternately constructs two disjoint candidate solutions within a constant number of sequence rounds. Then, the algorithm boosts solution quality without sacrificing the adaptive complexity. Extensive experimental studies on three applications, Revenue Maximization, Image Summarization, and Maximum Weighted Cut, show that our algorithm not only significantly increases solution quality but also requires comparative adaptivity to state-of-the-art algorithms.
	\end{abstract}
	
	\section{Introduction}
	A wide range of instances in artificial intelligence and machine learning have been modeled as a problem of Submodular Maximization under Knapsack constraint ($\SMK$) such as maximum weighted cut \cite{Amanatidis_sampleGreedy,Han2021_knap}, data summarization \cite{Han2021_knap,fast_icml}, revenue maximization in social networks \cite{Han2021_knap,Cui-aaai23,Cui-streaming21}, recommendation systems \cite{Amanatidis2021a_knap,Amanatidis_sampleGreedy}. The attraction of this problem comes from the diversity of submodular utility functions and the generalization of the knapsack constraint. The submodular function has a high ability to gather a vast amount of information from a small subset instead of extracting a whole large set, while the knapsack constraint can represent the budget, the cardinality, or the total time limit for a resource. Hence, people are interested in proposing expensive algorithms for $\SMK$ these years~\cite{Amanatidis2021a_knap,Han2021_knap,Cui-aaai23,pham-ijcai23,Amanatidis_sampleGreedy}. 
	
	
	Formally, a $\SMK$ problem can be defined such as given a ground set $V$ of size $n$, a budget $B>0$, and a non-negative submodular set function (not necessary monotone) $f: 2^V \mapsto \mathbb{R}_+$. Every element $e\in V$ has its positive cost $c(e)$. The problem $\SMK$ asks to find $S\subseteq V$ subject to $c(S)=\sum_{e \in S}c(e)\leq B$ that maximizes $f(S)$.
	
	One of the main challenges of $\SMK$ is addressing big data in which the sizes of applications can grow exponentially. The modern approach is to design approximation algorithms with low query complexity representing the total number of queries to the oracle of $f$. 
	However, required oracles of $f$ are often expensive and may take a long time to process on the machine within a single thread. Therefore, people think of designing efficient parallel algorithms that can leverage parallel computer architectures to obtain a good solution promptly.
	This motivates the \textit{adaptive complexity or adaptivity}~\cite{BalkanskiS18-adapt-stoc18} to become an important measurement of parallel algorithms. It is defined as the number of sequential rounds needed if the algorithm can execute polynomial independent queries in parallel. Therefore, the lower the adaptive complexity of an algorithm is, the higher its parallelism is.
	
	In the era of big data now, several algorithms that achieve near-optimal solutions with low adaptive complexities have been developed recently (See Table~\ref{tab:1} for an overview of low adaptive algorithms). 
	As can be seen, although recent studies make an outstanding contribution by significantly reducing the adaptive complexity of a constant factor approximation algorithm from $O(\log^2 n)$ to $O(\log n)$, there are two drawbacks, including (1) the high query complexities make them become impractical in some instances~\cite{Ene-icml20} and (2) there is a huge gap between the high approximation factors of low adaptivity algorithms, e.g. \cite{Amanatidis2021a_knap,Cui-aaai23,Cui-aaai23-full}, compared to the best one, e.g. \cite{BuchbinderF19-bestfactor}. This raises to us an interesting question: \textit{Is it possible to improve the factor of an approximation algorithm with near-optimal adaptive complexity of $O(\log n)$?} 
	\begin{table*}[ht!]
		\centering
		\begin{tabular}{llll}
			\hline
			\textbf{Reference} & \textbf{Approximation Factor} & \textbf{Adaptive Complexity} & \textbf{Query Complexity}
			\\
			\hline
			\cite{BuchbinderF19-bestfactor} & $2.6$ &$poly(n)$ &$poly(n)$
			\\
			\cite{Han2021_knap} & $4+\epsilon$ &$O(n\log k)$ &$O(n \log k)$	
			\\
			\cite{pham-ijcai23} & $4+\epsilon$ &$O(n)$ &$O(n)$
			\\
			\cite{Ene-stoc19} & $e+\epsilon$ &$O(\log^2n)$ &$\sO(n^2)$ 
			\\	
			\cite{Amanatidis2021a_knap} & $9.465+\epsilon$  &{\boldmath$O(\log n)$} &$\sO(n^2)$ 
			\\
			\cite{Cui-aaai23} (Alg.3) & $8+\epsilon$ &{\boldmath$O(\log n)$} &{\boldmath$\sO(nk) $}
			\\
			\cite{Cui-aaai23} (Alg.5) & $5+2\sqrt{2}+\epsilon\approx 7.83+\epsilon$ &$O(\log^2n)$ &{\boldmath$\sO(nk) $}
			\\
			$\AST$ \textbf{(Algorithm~\ref{alg:ast}, this work)} & {\boldmath$ 7+\epsilon$} &{\boldmath$O(\log n)$} &{\boldmath$\sO(nk) $}
			\\
			\hline
		\end{tabular}
		\caption{Algorithms for $\SMK$ problem. We use the $\sO$ notation throughout the paper to hide $poly(\log n)$ factors and $k$ is the largest cardinality of any feasible solution.  Bold font indicates the best result(s) in each setting.}
		\label{tab:1}
	\end{table*}
	\paragraph{Our contributions.} In this work, we address the above question by introducing the $\AST$ algorithm for the non-monotone $\SMK$ problem. $\AST$ has an approximation factor of $7+\epsilon$, within a pair of $O(\log n)$ adaptivity, $\sO(nk)$ query complexity, where $\epsilon$ is a constant input. Therefore, our algorithm improves the best factor of the near-optimal adaptive complexity algorithm in \cite{Cui-aaai23}. 
	We investigate the performance of our algorithm on three benchmark applications: Revenue Maximization, Image Summarization, and Maximum Weighted Cut. The results show that our algorithm not only significantly
	improves the solution quality but also requires comparative adaptivity to existing practical algorithms.	
	\paragraph{New technical approach.} It is noted that one popular approach to designing parallel algorithms with near-optimal adaptivity of $O(\log n)$ is based on making multiple guesses of the optimal solution in parallel and adapting a threshold sampling method\footnote{We refer to threshold sampling methods as $\ThS$ in \cite{Amanatidis2021a_knap} and $\Rand$ in \cite{Cui-aaai23} with $O(\log n)$ adaptivity.}, which selects a batch of elements whose density gains, i.e., the ratio between the marginal gain of an element per its cost, are at least a given threshold within $O(\log n)$ adaptivity \cite{Amanatidis2021a_knap,Cui-aaai23}. By making the guesses of the optimal along with calling the threshold sampling multiple times in parallel, the existing algorithms could keep the adaptive complexity of $O(\log n)$ and obtain some approximation ratios. 
	
	From another view, we introduce a novel algorithmic framework named ``alternate threshold'' to improve the approximation factor to $7+\epsilon$ but keep the same adaptivity and query complexity with the best one~\cite{Cui-aaai23}. Firstly, we adapt an existing adaptive algorithm to find a near-optimal solution within $O(\log n)$ adaptivity and give a $O(1)$ number of guesses of the optimal solution. Then, the core of our framework consists of a constant number of iterations. 
	It initiates two disjoint candidate sets and then adapts the threshold sampling to upgrade them alternately during iterations: one is updated at \textit{odd} iterations, and another is updated at \textit{even} iterations. Thanks to this strategy, we can find the connection between two solutions for supporting each other in evaluating the ``utility loss" after each iteration.
	At the end of this stage, we enhance the solution quality by finding the best element to be added to each candidate's subsets (pre-fixes of $i$ elements) without violating the budget constraint.
	
	It must be noted that our method differs from the Twin Greedy-based algorithms~\cite{Han_twinGreedy,pham-ijcai23,best-dla}, which update both candidate sets at the same iterations but do not allow the integration of the threshold sampling algorithm for parallelization.
	Besides, we carefully analyze the role of the highest cost element in the optimal solution to deserve more tightness for the problem.
	\section{Related Works}
	\label{sec:relatedwork}
	This section focuses on the related works for the non-monotone case of the $\SMK$ problem. 
	
	Firstly, regarding the non-adaptive algorithms, the first algorithm for the non-monotone $\SMK$ problem was due to \cite{Lee_nonmono_matroid_knap} with the $5+\epsilon$ factor and polynomial query complexity. Later, several works concentrated on improving both approximation factor and query complexity \cite{BuchbinderF19-bestfactor,Gupta_nonmono_constrained_submax,fast_icml,nearly-liner-der-2018,best-dla,pham-ijcai23,Han2021_knap}. In this line of works, algorithm of \cite{BuchbinderF19-bestfactor} archived the best approximation factor of $2.6$ but required a high query complexity; the fastest algorithm was proposed by \cite{pham-ijcai23} with $4+\epsilon$ factor in linear queries. 
	For the non-monotone Submodular Maximization under Cardinality ($\SMC$) problem, which finds the best solution that does not exceed $k$ elements to maximize a submodular objective value, the best factor of $2.6$ of the algorithm in \cite{BuchbinderF19-bestfactor} still held. Besides, a few algorithmic models have been proposed for improving running time~\cite{fast_sub,Kuhnle2021_mono,li-linear-knap-nip22,Buchbinder_Qrtradeoff}. Among them, the fastest algorithm belonged to~\cite{Buchbinder_Qrtradeoff} that provided a $e+\epsilon$ factor within $O(n\log(1/\epsilon)/\epsilon^2)$ queries. However, the above approaches couldn't be parallelized efficiently by the high adaptive complexity of $\Omega(n)$.

	The adaptive complexity was first proposed by~\cite{BalkanskiS18-adapt-stoc18} for the $\SMC$ problem. Regarding adaptivity-based algorithms for non-monotone $\SMK$, the first one belonged to~\cite{Ene-soda19} with $e+\epsilon$ and $O(\log^2n)$ adaptive complexity. However, due to the high query complexity of accessing, the multi-linear extension of a submodular function and its gradient in their method becomes impractical in real applications~\cite{Amanatidis2021a_knap,Fahrbach-icml19}. After that, \cite{Amanatidis2021a_knap} devised a $(9.465 + \epsilon)$-approximation algorithm within $O(\log n)$, which was optimal up to a $\Theta(\log \log (n))$ factor by adopting the lower bound in \cite{BalkanskiS18-adapt-stoc18}. It is noted that improving the adaptive complexity of a constant factor algorithm from $O(\log^2 n)$ to $O(\log n)$ made an outstanding contribution since it greatly reduced the number
	of sequential rounds in practical implementation \cite{Cui-aaai23,Ene-icml20,Fahrbach-icml19}.
	More recently, \cite{Cui-aaai23} created a big step when contributing an efficient parallel one, which resulted in a factor of $8+\epsilon$ within a pair of $O(\log n)$ adaptivity and $\sO(nk)$ query complexity. 
	Nevertheless, this factor still has a huge gap with the best factor of $2.6$~\cite{BuchbinderF19-bestfactor}. They also provided an enhanced version of increasing the approximation factor to $5+2\sqrt{2}+\epsilon\approx7.83+\epsilon$. However, it required a higher adaptivity of $O(\log^2n)$. Thus, from this view, the above result of \cite{Cui-aaai23} is the best one until now.
	
	People have also focused on developing parallel algorithms for non-monotone $\SMC$ these years~\cite{Kuhnle21_nonmono_adapt,Ene-icml20,Fahrbach-icml19}, etc. The fore-mentioned contributions of \cite{Ene-icml20}
	was also applied for $\SMC$ to get the best approximation factor of $e+\epsilon$, however it used multi-linear extension and thus had
	a high query complexity.
	Next, \cite{Kuhnle21_nonmono_adapt} and \cite{Fahrbach-icml19} tried to reduce the adaptive complexity to $O(\log n)$ with $25.64+\epsilon$ and $6+\epsilon$ factors. However, \cite{Khunle-non-mono-revised} claimed that both \cite{Kuhnle21_nonmono_adapt} and \cite{Fahrbach-icml19} had a non-trivial error because they used the same threshold sampling subroutine which did not work for the non-monotone objective function. \cite{Khunle-non-mono-revised} further tried to fixed the previous work and recovered the $6+\epsilon$ factor in $O(\log n)$.
	Recently, \cite{Amanatidis2021a_knap} improved the factor to $5.83+\epsilon$ in $O(\log n)$ adaptivity.
	Later, the work of \cite{Cui-aaai23} archived the factor of $8+\epsilon$ in $O(\log n)$ adaptivity or $4+\epsilon$ factor in $O(\log^2n)$.
	
	After all, our algorithm overcome the existing drawbacks by an improved parallel version with the approximation factor increasing to $7+\epsilon$ within $O(\log n)$ rounds to parallel $\sO(nk)$ queries. 

	\section{Preliminaries}
	\label{sec:preli}
	Given a ground set $V = \{e_1,\ldots, e_n\}$ and an utility set function $f: 2^V \mapsto \mathbb{R}_+ $ to measure the
	quality of a subset $S \subseteq V$, we use the definition of submodularity based on \textit{the diminishing return property}: $f: 2^V \mapsto \mathbb{R}_+ $. $f$ is submodular iff for any $A \subseteq B \subseteq V$ and $e \in V\setminus B$, we have 
	\begin{align*}
	f(e|A) \geq f(e|B).
	\end{align*}
	Each element $e \in V$ is assigned a positive cost $c(e)>0$. Let $c: 2^V \mapsto \mathbb{R}^+ $ be a cost function. Assume that $c$ is modular, i.e.,  $c(S) =\sum_{e \in S}c(e)$ such that $c(S)=0$ iff $S=\emptyset$. 
	
	The problem $\SMK$ asks to find $S\subseteq V$ subject to $c(S)=\sum_{e \in S}c(e)\leq B$ that maximizes $f(S)$.
	We denote by a tuple $(f, V, B)$ an instance of $\SMK$. Without loss of generality, $f$ is assumed non-negative, i.e., $f(X)\geq 0$ for all $X\subseteq V$ and normalized, i.e., $f(\emptyset)=0$. We also assume there exists an oracle query, which, when queried with the
	set $S$ returns the value $f(S)$. 
	
	For convenience, we denote by $S\cup e$ as $S \cup \{e\}$. Next, we denote by $O$ an optimal solution with the optimal value $\opt=f(O)$ and $r=\arg\max_{o \in O}c(o)$. We also define the contribution gain of a set $T$ to a set $S$ as $f(T|S)=f(T\cup S)-f(S)$. Also, the contribution gain of an element $e$ to a set $S\subseteq V$ is defined as $f(e|S)=f(S\cup \{e\})-f(S)$ and $f(\{e\})$ is written as $f(e)$ for any $e\in V$.
	
	In this paper, we design a parallel algorithm based on \textit{Adaptive complexity} or \textit{Adaptivity}, which is defined as follows:
	\begin{definition}[Adaptive complexity or Adaptivity \cite{BalkanskiS18-adapt-stoc18}] Given a value of oracle of $f$, the adaptivity or adaptive complexity of an algorithm is the minimum number of rounds needed such that in each round, the algorithm makes $O(poly(n))$ independent queries to the evaluation oracle. 
	\end{definition}
	In the following, we recap two sub-problems which our algorithm need to solve: Unconstrained Submodular Maximization and Density Threshold.%
	\paragraph{Unconstrained Submodular Maximization ($\USM$)} This problem requires to find a subset $S \subseteq V$ that maximizes $f(S)$ without any constraint. The problem was shown NP-hard \cite{Uriel-siam11}. 
	\\
	To obtain mentioned approximation factor, our algorithm adapts the low adaptivity algorithm in \cite{chen_stoc19} that achieves an approximation factor of $(2+\epsilon)$ in constant adaptive rounds of $O(\log(1/\epsilon)/\epsilon)$ and linear queries of $O(n\log^3(1/\epsilon)/\epsilon^4)$. 
	\paragraph{Density Threshold (DS).} The problem receives an instance $(f, V, B)$, a fixed threshold $\tau$ and a parameter $\epsilon>0$ as inputs, it asks to find a subset $S\subseteq V$ satisfies two conditions: (1) $f(S)\geq c(S)\cdot \tau$; (2) $\sum_{e\in V\setminus S}f(e|S)\leq \epsilon \cdot \opt$.
	\\
	Two algorithms in the literature satisfy the above conditions,
	including those in \cite{Amanatidis2021a_knap} and \cite{Cui-aaai23}. 
	In this work, we adapt the $\Rand$ algorithm in \cite{Cui-aaai23}. $\Rand$ requires the set $I$, a submodular function $f(\cdot)$, and parameters $\epsilon, M$ to control the solution's accuracy and complexities. $\Rand$ is combined with the aforementioned density thresholds to set up sieves in parallel for $\SMK$. Due to the space limitations, Pseudocode for $\Rand$ is given in the appendix.   
	
	For an instance $(V, f, B)$ of $\SMK$, two subsets $I, M$ of $V$, a fixed threshold $\theta$ and   input parameter  $\epsilon$. 
	The performance of $\Rand$ is provided in the following Lemmas.
	\begin{lemma}[Lemma 1 in~\cite{Cui-aaai23}]
		The sets $A$, $L$ output by $\Rand(\theta, I, M, \epsilon, f(\cdot), c(\cdot))$ satisfy $\E[f(A)]\geq (1-\epsilon)^2\theta \cdot \E[c(A)]$ and $\epsilon \cdot M \cdot \sum_{u\in L}f(u|A)\leq \opt$.
		\label{lem:cost-ex}
	\end{lemma}
	\begin{lemma}[Lemma 2 in~\cite{Cui-aaai23}]
		$\Rand$ has $O(\frac{1}{\epsilon p}( \log(|I| \cdot\beta(I)) + M))$
		adaptivity, and its query complexity is $O(|I|k)$ times of
		its adaptive complexity, where $\beta(I)=\max_{u, v}\frac{c(u)}{c(v)}$. If
		we use binary search in Line 10 of $\Rand$, then it has
		$O(\frac{1}{\epsilon p}(\log(|I| \cdot\beta(I)) + M)\log(k))$ adaptivity, and its
		query complexity is $O(|I|)$ times of its adaptivity.
	\end{lemma}
	Interestingly, we further explore a useful property of $\Rand$ when applying it to our algorithm.
	\begin{lemma}
		The sets $A$, $L$ output by $\Rand(\theta, I, M, \epsilon, f(\cdot), c(\cdot))$ satisfy $\E[f(a_i|A_{i-1})]\geq (1-\epsilon)^2\E[c(a_i)]\theta$, where $A=\{a_1, a_2, \ldots, a_{|A|}\}, A_i=\{a_1, a_2, \ldots, a_i\}$.
		\label{lem:add-element}
	\end{lemma}	
	
	
	
	\section{Proposed Algorithm}
	\label{sec:alg}
	In this section, we introduce $\AST$ (Algorithm~\ref{alg:ast}), a $(7+\epsilon)$-approximation algorithm in $O(\log n)$ adaptivity and $O(n^2 \log^2n)$ query complexity. 
	
	$\AST$ receives an instance $(f, V, B)$, constant parameters $\delta, \epsilon, \alpha$ as inputs.
	It contains two main phases. At the first phase (Lines \ref{algo-ast:ph1-1}-\ref{algo-ast:ph1-2}), it first divides the ground set $V$ into two subsets: $V_0$ contains elements with small costs, and $V_1$ contains the rest. $\AST$ then calls $\PAR$~\cite{Cui-aaai23} as a subroutine which returns a $(1/8-\delta)$-approximation solution within $O(\log n)$ adaptive rounds. Based on that, the algorithm can offer $O(\log(1/\epsilon)/\epsilon)$ guesses of the optimal solution for the main loop (Lines 4-\ref{algo-ast:ph1-2}).
	The main loop consists of $O(\log(1/\epsilon)/\epsilon)$ iterations; each corresponds to a guess of the optimal. It sequentially constructs two disjoint solutions, $X$ and $Y$, one at odd iterations and the other at even iterations. The work of the odd and the even is the same. At the odd (or even) ones, it sets the threshold $\theta_X$ ($\theta_Y$) and calls the $\Rand$ routine with the ground set $I$ and the function $f(\cdot|X)$ ($f(\cdot|Y)$) as inputs to provide the new set $A_i$ ($B_i$) (Lines 8, 12). 
	It then updates $X$ ($Y$) and $I$ as the remaining elements (Line 8 or 12). 
	
	The second phase (Lines 15-24) is to improve the quality of solutions. If $c(X_1 \cup V_0)\leq \epsilon B$, this phase first adapts $\USM$ algorithm \cite{chen_stoc19} for unconstrained submodular maximization over $X_1 \cup V_0$ to get a candidate solution $S_1$ (Lines 15-16). This step is based on an observation that $X_1$ is important in analyzing the algorithm's performance. It then selects the sets of the first $i$ elements added into $X$ and $Y$ and finds the best elements without violating the total cost constraint (Lines 19-24). Finally, the algorithm returns the best candidate solution (Lines 25-26). The details of $\AST$ are depicted in Algorithm~\ref{alg:ast}.
	
	\begin{algorithm}[h]
		\SetNlSty{text}{}{:}
		\KwIn{An instance $(f, V, B)$, parameters $\alpha$, $\epsilon$, $\delta$}
		$V_0 \leftarrow \{e\in V: c(e)\leq \epsilon B/n\}$, $V_1 \leftarrow V \setminus V_0$, $I \leftarrow V_1$, $p\leftarrow1$ \label{algo-ast:ph1-1}
		\\
		$S_0 \leftarrow \PAR(\frac{1}{4},\delta, f(\cdot), c(\cdot))$, $\Gamma \leftarrow \frac{8 \alpha f(S_0)}{(1-8\delta)\epsilon B}$
		\\
		$X\leftarrow \emptyset, Y\leftarrow \emptyset$, $\Delta \leftarrow \lceil \log_{\frac{1}{1-\epsilon}}\frac{8\alpha}{\epsilon^2(1-8\delta)} \rceil+1$, $M \leftarrow \frac{1}{\epsilon^2}(\frac{\Delta}{2}+1) $
		\\
		\For{$i=1$ to $\Delta$}
		{
			\eIf{$i$ is \textbf{odd}}
			{
				$\theta^X \leftarrow \Gamma(1-\epsilon)^{i}$
				\\
				$ (A_i, U_i, L_i) \leftarrow \Rand(\theta^X, I, M,p, \epsilon, f(\cdot|X), c(\cdot))$
				\\
				$X\leftarrow X\cup A_i$, $I \leftarrow I\setminus X$ \label{alg-ast:update}
			}	
			{
				$\theta^Y \leftarrow \Gamma(1-\epsilon)^i$
				\\
				$ (B_i, U_i, L_i) \leftarrow \Rand(\theta^Y, I, M,p, \epsilon, f(\cdot|Y), c(\cdot))$
				\\
				$Y\leftarrow Y\cup B_i$, $I \leftarrow I\setminus Y$
			}
		}
		\label{algo-ast:ph1-2}
		For $T\in \{X, Y\}$,	define:
		$T_i$ is $T$ after the iteration $i$ of the first loop,
		$T^i$ is the set of first $i$ elements added into $T$.
		\\
		\If{$c(X_1 \cup V_0) \leq \epsilon B$} 
		{ \label{line-1}
			$S_1 \leftarrow \USM(X_1 \cup V_0)$
		}\label{line-2}
		\For{$i=1$ to $|X|$}{
			$a_i\leftarrow\arg\max_{e\in V: c(X^i\cup e)\leq B}f(X^i\cup \{e\})$, 
			$X'^i\leftarrow X^i\cup \{a_i\}$
		}
		\For{$i=1$ to $|Y|$}{
			$b_i\leftarrow\arg\max_{e\in V: c(Y^i\cup e)\leq B}f(Y^i\cup \{e\})$,
			$Y'^i\leftarrow Y^i\cup \{a_i\}$ 
		}
		$S \leftarrow \arg\max_{T\in \{X'^i\}^{|X|}_{i=1} \cup \{Y'^i\}^{|Y|}_{i=1}\cup \{X, Y, S_1\} }f(T)$
		\label{algo-ast:ph2-1}
		\\
		\Return $S$
		\caption{$\AST$ Algorithm}
		\label{alg:ast}
	\end{algorithm}
	
	At the high level, $\AST$ works follow a novel framework that combines an alternate threshold greedy algorithm with the boosting phase. 
	The term ``alternate'' means that candidate solutions are updated alternately with each other in multiple iterations. At each iteration, only one partial solution is updated based on two factors: one guess of the optimal solution and the remaining elements of the ground set that do not belong to the other solution. 
	
	It should be emphasized that the alternate threshold greedy differs from recent works ~\cite{Cui-aaai23,Amanatidis2021a_knap} where two candidate solutions for each guess are constructed after only one adaptive round. Alternate threshold greedy also differs from the twin greedy method in \cite{Han_twinGreedy}, which allows updating both disjoint sets in each iteration. 
	For the theoretical analysis, the key to obtaining a tighter approximation factor lines in aspects: (1) the connections between $X$ and $Y$ after each iteration of the first loop and (2) carefully considering the role of $r$ to eliminate terms that worsen the approximation factor.
	
	We now analyze the performance guarantees of $\AST$. We consider $X$ and $Y$ after ending the first loop. We first introduce some notations regarding $\AST$ as follows. 
	\begin{itemize}
		\item $X^i$, $Y^i$ is the set of first $i$ elements added into $X$ and $Y$, respectively.
		\item $X_i$ and $Y_i$ are $X$ and $Y$ after the iteration $i$ of the first loop (Lines 4-14) and $X_0=Y_0=\emptyset$.
		\item $O_1$ is an optimal solution of $\SMK$ over instance $(f, V_1, B)$.
		\item $O'=O_1\setminus X_1, O^r=O_1\setminus \{r\}$ and $O'^r=O'\setminus \{r\}$.
		\item For an element $e \in X\cup Y$, we denote: $X_{<e}$, $Y_{<e}$, $\theta^X_e$ and $\theta^Y_e$ as $X$, $Y$, $\theta^X$ and $\theta^Y$ right before $e$ is selected into $X$ or $Y$, respectively; $l(e)$ is the iteration when $e$ is added in to $X$ or $Y$. 
	\end{itemize}
	Lemma~\ref{lem:bound-T} makes a connection between $X$ and $Y$ after each iteration.
	\begin{lemma} 
		After any iteration $i$ of the first loop (Lines 4-14) of $\AST$, we have:
		\begin{itemize}
			\item [a)] If $i\geq 1$, $i$ is odd and $c(X_i)\leq B- c(r)$. Let $T\subseteq Y_{i-1}\cap O_1$, we have 	$\sum_{e \in T}f(e|X_i) < \sum_{e\in T}\frac{\E[f(e|Y_{<e})]}{(1-\epsilon)^3}+ \epsilon \cdot \opt$.
			\item[b)] If $i\geq 2$, $i$ is even and $c(Y_i)\leq B-c(r)$. Let $T\subseteq X_{i-1}\cap O'$, we have 	$\sum_{e \in T}f(e|Y_i) < \sum_{e\in T}\frac{\E[f(e|X_{<e})]}{(1-\epsilon)^3}+ \epsilon \cdot \opt$.
		\end{itemize}
		\label{lem:bound-T}
	\end{lemma}
	\begin{proof}
		\textbf{Prove a)} If $i=1$, $Y_0=\emptyset$, the Lemma holds. We consider the other case. We divide $T$ into several subsets including $T=T_2\cup T_4\cup \ldots\cup T_{i-1}$, where $T_j$ is a set of all elements in $T$ that are added into $Y_i$ at iteration $j\leq i-1$. Since $T_j\subseteq Y_{i-1}\cap O_1$ and $c(X_i)\leq B-c(r)$ so $c(X_{<e})+c(e)\leq c(X_i)+c(e)\leq B$ for all $e \in T_j$. 
		We therefore classify the elements in $T_j$ into two disjoint sets $T_j=T_j^1\cup T_j^2$, where $T_j^1=\{e\in T_j:\frac{f(e|X_{<e})}{c(e)}<\theta^X_{e}, c(X_{<e})+c(e)\leq B \}$ and $T_j^2=\{e\in T_j:\frac{f(e|X_{<e})}{c(e)}\geq \theta^X_{e}, c(X_{<e})+c(e)\leq B\}$. Since $(1-\epsilon)\theta^X_e = \theta^Y_{l(e)} \forall e\in T^1_j$, we have: 
		\begin{align}
		\sum_{e\in T_j}f(e|X_{<e})&= \sum_{e\in T^1_j}f(e|X_{<e})+ \sum_{e\in T^2_j}f(e|X_{<e}) 
		\\
		&< \sum_{e\in T^1_j}c(e)\theta^X_{e}+ \frac{ \opt}{\epsilon M} \label{ineq1}
		\\
		& =\sum_{e\in T^1_j}c(e)\frac{\theta^Y_{l(e)}}{1-\epsilon}+\frac{ \opt}{\epsilon M}
		\\
		&\leq \sum_{e\in T^1_j}\frac{\eE[f(e|Y_{<e})]}{(1-\epsilon)^3}+\frac{ \opt}{\epsilon M} \label{ineq2}
		\end{align}
		where inequality~\eqref{ineq1} is due to Lemma~\ref{lem:cost-ex}, inequality~\eqref{ineq2} is due to applying Lemma~\ref{lem:add-element}:
		$
		\eE[f(e|Y_{<e})]\geq (1-\epsilon)^2\E[c(e)]\theta^Y_{l(e)}
		$. It follows that
		\begin{align}
		\sum_{e \in T}f(e|X_i)&=\sum_{j=2, 4, \ldots, i-1}\left( \sum_{e\in T_j}f(e|X_i)\right)
		\\
		& \leq\sum_{j=2, 4, \ldots, i-1}\left( \sum_{e\in T_j}f(e|X_{<e})\right) \label{ine3}
		\\
		& <\sum_{j=2, 4, \ldots, i-1}\left( \sum_{e\in T^1_j}\frac{\E[f(e|Y_{<e})]}{(1-\epsilon)^3}+ \frac{\opt}{\epsilon M}\right) \nonumber
		\\
		&	\leq \sum_{e\in T}\frac{\E[f(e|Y_{<e})]}{(1-\epsilon)^3}+ (\frac{\Delta}{2}+1) \cdot \frac{\opt}{\epsilon M}
		\\
		&	\leq \sum_{e\in T}\frac{\E[f(e|Y_{<e})]}{(1-\epsilon)^3}+ \epsilon \opt \label{ine4} 
		\end{align}
		where inequation~\eqref{ine3} is due to the submodularity, inequation~\eqref{ine4} is due to setting of $M$.
		\\
		\textbf{Prove b)}. If $i=2$, $X_1\cap O=\emptyset$, the Lemma holds. If $i>2$, we only consider set $T\subseteq X_{i-1}\cap O'$ since we can not bound the $f(e|Y_{<e})$ if $e \in X_1$.
		We also divide $T$ into subsets: $T=T_3\cup T_5\cup \ldots\cup T_{i-1}$, where $T_j$ is a set added into $X_j$ at the iteration $j\leq i-1$. We also classify the elements in $T_j$ into two sets $T_j=T_j^1\cup T_j^2$, where $T_j^1=\{e\in T_j:\frac{f(e|Y_{<e})}{c(e)}<\theta^Y_{<e}, c(Y_{<e})+c(e)\leq B\}$ and $T_j^2=\{e\in T_j:\frac{f(e|Y_{<e})}{c(e)}\geq \theta^Y_{<e}, c(Y_{<e})+c(e)\leq B\}$. By a similar argument to the previous case, we have:
		\begin{align}
		\sum_{e\in T_j}f(e|Y_{<e})&= \sum_{e\in T^1_j}f(e|Y_{<e})+ \sum_{e\in T^2_j}f(e|Y_{<e}) 
		\\
		&< \sum_{e\in T^1_j}c(e)\theta^Y_{e}+ \frac{\opt}{ \epsilon M }
		\\
		& =\sum_{e\in T^1_j}c(e)\frac{\theta^X_{e}}{1-\epsilon}+\frac{\opt}{\epsilon M }
		\\
		&\leq \sum_{e\in T^1_j}\frac{\E[f(e|X_{<e})]}{(1-\epsilon)^3}+\frac{\opt}{\epsilon M } 
		\end{align}
		which implies that 
		\begin{align}
		\sum_{e \in T}f(e|Y_i)&=\sum_{j=3, 5, \ldots, i-1}\left( \sum_{e\in T_j}f(e|Y_i)\right)
		\\
		& \leq\sum_{j=3, 5, \ldots, i-1}\left( \sum_{e\in T_j}f(e|Y_{<e})\right) 
		\\
		& <\sum_{j=3, 5, \ldots, i-1}\left( \sum_{e\in T^1_j}\frac{\eE[f(e|X_{<e})]}{(1-\epsilon)^3}+ \frac{\opt}{\epsilon M}\right) \nonumber
		\\
		&	\leq \sum_{e\in T}\frac{\eE[f(e|X_{<e})]}{(1-\epsilon)^3}+ (\frac{\Delta}{2}+1) \cdot \frac{\opt}{\epsilon M}
		\\
		&	\leq \sum_{e\in T}\frac{\E[f(e|X_{<e})]}{(1-\epsilon)^3}+ \epsilon \opt.
		\end{align}
		The proof is completed.
	\end{proof}
	By using Lemma~\ref{lem:bound-T}, we further
	provide the bound of $f(O'\cup T)$ for $T$ is a subset of $X$ or $Y$ in Lemma~\ref{lem:r-big} when $c(r)$ is very large, i.e., $c(r)\geq (1-\epsilon)B$.
	\begin{lemma}
		If $c(r)\geq (1-\epsilon) B$, one of two following propositions happens
		\begin{itemize}
			\item[a)] $\E[f(S)]\geq (1-\epsilon)^5\alpha\opt$.
			\item[b)] There exists a subset $X' \subseteq X$ so that
			$$f(X'\cup O')<(1+\frac{1}{(1-
				\epsilon)^3})\E[f(S)]+(2\epsilon+\frac{1}{\epsilon M})\opt.$$
		\end{itemize}
		Similarly, one of two following propositions happens:
		\begin{itemize}
			\item[c)] $\E[f(S)]\geq (1-\epsilon)^5\alpha\opt$.
			\item[d)] There exists a subset $Y' \subseteq Y$ so that
			$$f(Y'\cup O')< (1+\frac{1}{(1-
				\epsilon)^3})\E[f(S)]+(2\epsilon+\frac{1}{\epsilon M})\opt.$$
		\end{itemize}
		\label{lem:r-big}
	\end{lemma}
	When $c(X_1) < \epsilon B$ and $c(Y_2)< \epsilon B$, it's easy to obtain the approximation factor due to $f(S)\geq \max\{f(X_1), f(Y_2)\}\geq \epsilon B \alpha\Gamma(1-\epsilon)^2$. Otherwise, we combine Lemma~\ref{lem:r-big} and the fact that $f(O')\leq f(O'\cup X)+f(O'\cup Y)$ to get the bound of $f(O')$ in Lemma~\ref{lem:XY-small}. The proofs of them can be found in the Appendix. 
	\begin{lemma}
		If $c(X_1) < \epsilon B$ and $c(Y_2)< \epsilon B$, we have: 
		\begin{itemize}
			\item If $c(r)<(1-\epsilon)B$, then 
			$f(O')\leq \frac{5\E[f(S)]}{(1-\epsilon)^4}+ \epsilon\opt$
			\item If $c(r)\geq (1-\epsilon)B$, one of two things happens: 
			\begin{itemize}
				\item[\textbf{a)}] $\E[f(S)]\geq (1-\epsilon)^5\alpha\opt.$
				\item[\textbf{b)}] $f(O')\leq \frac{5\E[f(S)]}{(1-\epsilon)^3}+2(2\epsilon+\frac{1}{\epsilon M})\opt$.
			\end{itemize}
		\end{itemize}
		\label{lem:XY-small}
	\end{lemma}
	Finally, put Lemmata \ref{lem:bound-T},\ref{lem:r-big},\ref{lem:XY-small} together and divide $O$ into appropriate subsets, we state the performance's algorithm in Theorem~\ref{theo:main1}.
	\begin{theorem} For $\alpha =\frac{1}{7}$, $\epsilon \in (0, \frac{1}{7})$, $\delta \in (0, \frac{1}{8})$, Algorithm~\ref{alg:ast} needs $O(\log n)$ adaptive complexity and $O(nk \log^2n)$ query complexity and returns a solution $S$ satisfying $\E[f(S)]\geq (1/7-\epsilon)\opt$.
		\label{theo:main1}
	\end{theorem}
	\begin{proof}
		$\AST$ first calls $\PAR$ to find a candidate solution $S_0$. This task takes $O(\frac{1}{\delta}\log(\frac{n}{\delta}))$ adaptivity and $O(nk\log^2n)$ query complexity \cite{Cui-aaai23-full}. For the first loop, it calls $\ThS$ $\Delta=O(\frac{1}{\epsilon}\log(\frac{1}{\epsilon}))$ times. Each time, $\Rand$ needs $O(\frac{1}{\epsilon}\log(\frac{n}{\epsilon})+ M)=O(\frac{1}{\epsilon}\log(\frac{n}{\epsilon})+\frac{1}{\epsilon^3}\log(\frac{1}{\epsilon}))=O(\log n)$ adaptive complexity and $O(\frac{nk}{\epsilon}\log(\frac{n}{\epsilon})+\frac{nk}{\epsilon^3}\log(\frac{1}{\epsilon}))=O(nk\log n)$ query complexity. 
		In the second phase, the algorithm may need $O(\frac{1}{\epsilon}\log(\frac{1}{\epsilon}))$ adaptivity and $O(\frac{n}{\epsilon^4}\log^3(\frac{1}{\epsilon}))$ query complexity to call $\USM$ algorithm of \cite{chen_stoc19} (Lines 16-17). Then, it only has two adaptive rounds and takes $O(kn)$ query complexity to find $X'^i$ and $Y'^i$ (Lines 19-24). Therefore, the adaptive complexity of the algorithm is 
		$
		O(\frac{1}{\delta}\log\frac{ n}{\delta})+O(\frac{1}{\epsilon}\log(\frac{1}{\epsilon})\log n)+ O(\frac{1}{\epsilon}\log(\frac{1}{\epsilon}))+2
		=O(\log n)
		$
		and its query complexity is
		$O(nk \log^2 n )+O(nk\log n)+O(\frac{n}{\epsilon^4}\log^3(\frac{1}{\epsilon}))+O(nk)
		=O(nk\log^2n)$.
		\\
		For the approximation factor, we consider the following cases:
		\\
		\textbf{Case 1.} If $c(X_1) \geq \epsilon B$ or $c(Y_2) \geq \epsilon B$, we have
		\begin{align*}
		f(S)&\geq \max\{f(X_1), f(Y_2)\}\geq \epsilon B \alpha\Gamma(1-\epsilon)^2
		\\
		& >\frac{(1-\epsilon)^2}{7}\opt>(\frac{1}{7}-\epsilon)\opt.
		\end{align*}
		\textbf{Case 2.} If $c(X_1) < \epsilon B$ and $c(Y_2) < \epsilon B$, we have
		\begin{align}
		f(O)&\leq f(O \cap (V_1 \setminus X_1))+f(O \cap (V_0\cup X_1)) \label{theo1}
		\\
		& = f(O')+f(O \cap (V_2\cup X_1)) \label{theo2}
		\\
		& \leq f(O') + (2+\epsilon)\E[f(S)] \label{theo3}
		\end{align}
		where inequality~\eqref{theo1} is due to the submodularity of $f$ and $(V_1\setminus X_1)\cap(V_0\cap X_1)=\emptyset$, equality~\eqref{theo2} is due to the definition of $O'$ and inequality~\eqref{theo3} is due to applying Algorithm in \cite{chen_stoc19}.
		We now apply Lemma~\ref{lem:XY-small} to bound $f(O')$. If $c(r)<(1-\epsilon)B$, then
		\begin{align}
		f(O)&\leq \frac{5\E[f(S)]}{(1-\epsilon)^4}+ \epsilon\opt+ (2+\epsilon)\E[f(S)] 
		\\
		& \leq\frac{7\E[f(S)]}{(1-\epsilon)^4}+\epsilon\opt.
		\end{align}
		Therefore 
		\begin{align*}
		\E[f(S)]\geq \frac{(1-\epsilon)^5 }{7}\opt>\frac{1-5\epsilon}{7}\opt> (\frac{1}{7}-\epsilon)\opt.
		\end{align*}
		If $c(r)\ge (1-\epsilon)B$, we consider two cases:
		\\
		- If \textbf{a)} in Lemma~\ref{lem:XY-small} happens, then
		$$
		\E[f(S)]\geq \frac{(1-\epsilon)^5}{7}\opt> (\frac{1}{7}-\epsilon)\opt.
		$$
		- If \textbf{b)} in Lemma~\ref{lem:XY-small} happens, then
		\begin{align}
		f(O)&\leq \frac{5\E[f(S)]}{(1-\epsilon)^3}+(4\epsilon+\frac{2}{\epsilon M})\opt
		+ (2+\epsilon)\E[f(S)] \nonumber
		\\
		& <\frac{7\E[f(S)]}{(1-\epsilon)^3}+ \frac{29}{7}\epsilon\opt. \label{theo4}
		\end{align}
		where the inequality~\eqref{theo4} is due to $\epsilon M =\frac{1}{\epsilon}(\frac{\Delta}{2}+1)>\frac{14}{\epsilon}$ for $\epsilon \in (0,\frac{1}{7}), \delta \in (0,\frac{1}{8})$. It follows that 
		\begin{align*}
		\E[f(S)] &\geq \frac{1}{7}(1-\epsilon)^3(1-\frac{29}{7}\epsilon)\opt 
		>(\frac{1}{7}-\epsilon)\opt
		\end{align*}
		which completes the proof.
	\end{proof}
	\section{Experimental Evaluation}
	\label{sec:expr}
	This section evaluates our $\AST$'s performance by comparing our algorithm with state-of-the-art algorithms for non-monotone $\SMK$ including:
	\begin{itemize}
		\item \textbf{$\PAR$}: The parallel algorithm in~\cite{Cui-aaai23} that runs in $O(\log n)$ adaptivity and returns a solution $S$ satisfying $\mathbb{E}[f(S)] \geq(1/8-\epsilon)\opt$.
		\item \textbf{$\PR$}: The algorithm in \cite{Cui-aaai23} that runs in $O(\log^2 n)$ adaptivity and returns a solution of $\mathbb{E}[f(S)] \geq\left(1/(5+2 \sqrt{2})-\epsilon\right)\opt$.
		\item \textbf{$\PRK$}: The parallel algorithm in \cite{Amanatidis2021a_knap} achieves an approximation factor of $(9.465+\varepsilon)$ within $O(\log n)$.
		\item \textbf{$\RaAc$}: The non-adaptive algorithm in \cite{Han2021_knap} that achieves an approximation factor of $4+\epsilon$ in query complexity of $O(n\log(k/\epsilon)/\epsilon)$.
		\item \textbf{$\RLA$}: The non-adaptive algorithm in \cite{pham-ijcai23} with a factor of $4+\epsilon$ in linear query complexity of $O(n \log (1 / \epsilon) / \epsilon)$.
	\end{itemize}
	We experimented with the following three applications:
	\paragraph{Revenue Maximization ($\RM$).}
	Given a network $G=(V, E)$ where $V$ is a set of nodes and $E$ is a set of edges. Each edge $(u,v)$ in $E$ is assigned a positive weight $w(u,v)$ sampled uniformly in $[0,1]$ and each node is assigned a positive cost $c(u)$ defined as $c(u)=$ $1- e^{\sqrt{\sum_{(u, v) \in E} w(u, v)}}$. The revenue of any subset $S \subseteq \mathcal{V}$ is defined as $f(S)=\sum_{v \in \mathcal{V} \backslash S} \sqrt{\sum_{u \in S} w_{u, v}}$. Given a budget of $B$, the goal of the problem is to select a set $S$ with the cost at most $B$ to maximize $f(S)$. As in the prior work, \cite{Amanatidis2021a_knap}, we can construct the graph $G$ using a YouTube community network dataset \cite{Han2021_knap} with 39,841 nodes and 224,235 edges.
	\paragraph{Maximum Weighted Cut ($\MWC$).} Consider a graph $G = (V, E)$ where each edge $(u,v) \in E$ has a non-negative weight $w(u,v)$. For a node subset $S\subset V$, define the weighted cut function $f(S) = \sum_{u\in V \setminus S} \sum_{v \in S} w(u,v)$. The maximum weighted cut problem seeks a subset $S\subseteq V$ that maximizes $f(S)$. As in recent work \cite{Amanatidis_sampleGreedy}, we generate an Erdős-Rényi random graph with $5,000$ nodes and an edge probability of $0.2$. The node costs $c(u)$ are randomly uniformly sampled from $(0,1)$.
	\paragraph{Image Summarization ($\IS$).} Consider a graph $G=(V, E)$ where each node $u\in V$ represents an image, and each edge $(u,v)\in E$ has a weight $w(u,v)$ showing the similarity between images $u$ and $v$. Let $c(u)$ be the cost to acquire image $u$. The goal is to find a representative subset $S\subseteq V$ under the budget $B$ that maximizes a value $f(S) = \sum_{u\in V} \max_{v\in S} w_{u,v} - \frac{1}{|V|} \sum_{u\in V}\sum_{v\in S} w_{u,v}$ \cite{fast_icml,Han2021_knap}. As in recent works \cite{Han2021_knap,fast_icml}, we create an instance as follows: First, randomly sample 500 images from the CIFAR dataset \cite{img-data} of 10,000 images, then measure the similarity between images $u$ and $v$ using the cosine similarity of their 3,072-dimensional pixel vectors.
	\\
	\textbf{Experiment setting.}	In our experiments, we set the accuracy parameter $\epsilon=0.1$ for all algorithms evaluated, and for $\AST$, we set $\delta=0.12$. We used OpenMP to program with C++ language. 
	Besides, we experimented on a high-performance computing (HPC) server cluster with the following parameters: partition=large, \#threads(CPU)=128, node=4, max memory = 3,073 GB. For $\USM$, we use setting of previous works \cite{Amanatidis2021a_knap,Cui-aaai23}, i.e, we adapt Algorithm in \cite{Feige-siamcomp11} returning $1/4-\epsilon$ ratio in one adaptive round and $O(n)$ query complexity.
	\paragraph{Experimental Result.} In Figures \ref{fig:res}(a), (c), and (e), we compare the objective values between different algorithms. The results show that our $\AST$ achieves the best objective values for both the $\RM$ and $\MWC$ applications. In $\RM$, the objective values achieved by $\RLA$, $\RaAc$, and $\PAR$ are the same, $\PRK$ attains lower objective values while $\PR$ hits the lowest objective values among the algorithms. Especially, our one marks the highest value when $B=0.015$, about $1.3$ times higher than the others in $\RM$. 
	In $\IS$, $\PRK$ reaches the highest values while $\PR$ hits the lowest. Our algorithm results in best values at some points and drops at others. As shown in Figure~\ref{fig:res} (c), most algorithms fluctuate widely. The variation in the quality of these algorithms might be due to the characteristics of this dataset.
	
	In Figures~\ref{fig:res}(b), (d), and (f), we make the comparisons about the number of adaptive rounds. The results show that $\RaAc$ and $\RLA$ always require the highest number of adaptive rounds across all three applications. For the $\AST$, the number of adaptive rounds is equivalent to $\PAR$ for $\RM$ and $\MWC$. Besides, for $\IS$, the adaptive number of rounds for $\AST$ is higher than that of $\PAR$, which has the lowest number of rounds. However, the higher number in this case is insignificant.
	Overall, our algorithm outperforms the others in both solution performance and the quantities of adaptivity.
	\begin{figure}[h]
		\centering
		{\includegraphics[width=.95\linewidth]{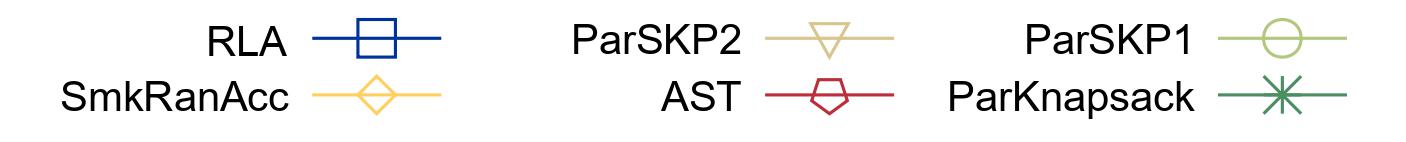}}
		\\
		{\includegraphics[width=.495\linewidth]{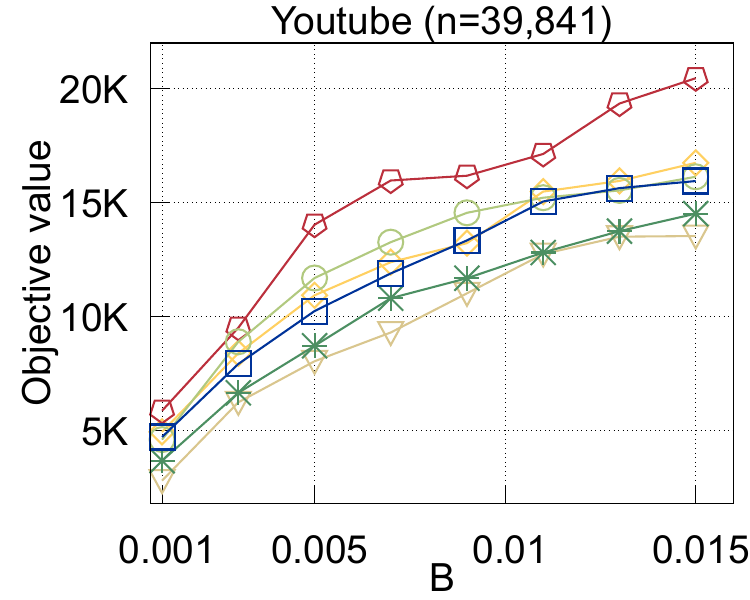}}
		{\includegraphics[width=.495\linewidth]{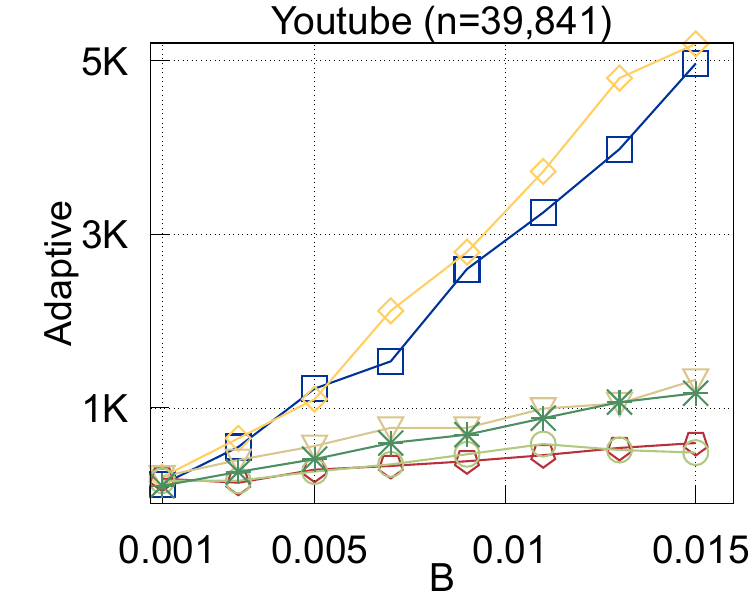}}
		\\ 
		\footnotesize{ \hspace{.5cm}	(a) \hspace{3.5cm} (b)} 
		\\
		{\includegraphics[width=.495\linewidth]{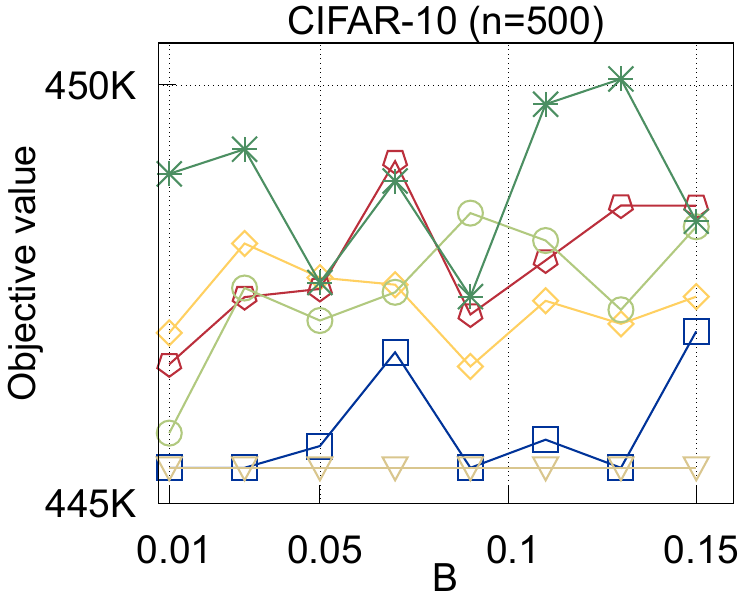}}
		{\includegraphics[width=.495\linewidth]{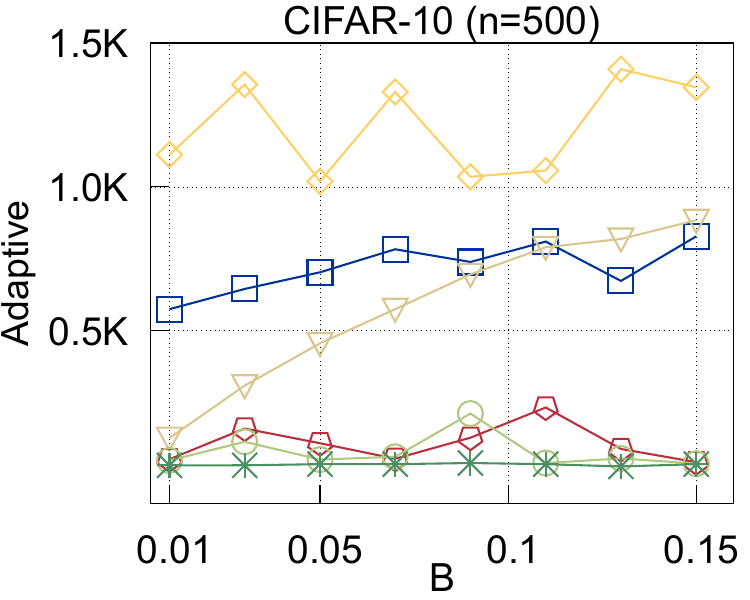}}
		\\
		\footnotesize{ \hspace{.5cm}	(c) \hspace{3.5cm}(d) }
		\\
		{\includegraphics[width=.495\linewidth]{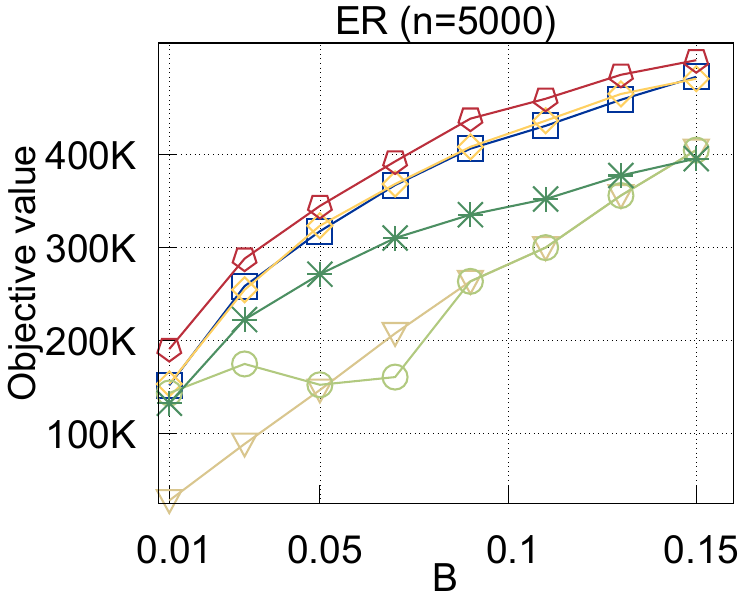}}
		{\includegraphics[width=.495\linewidth]{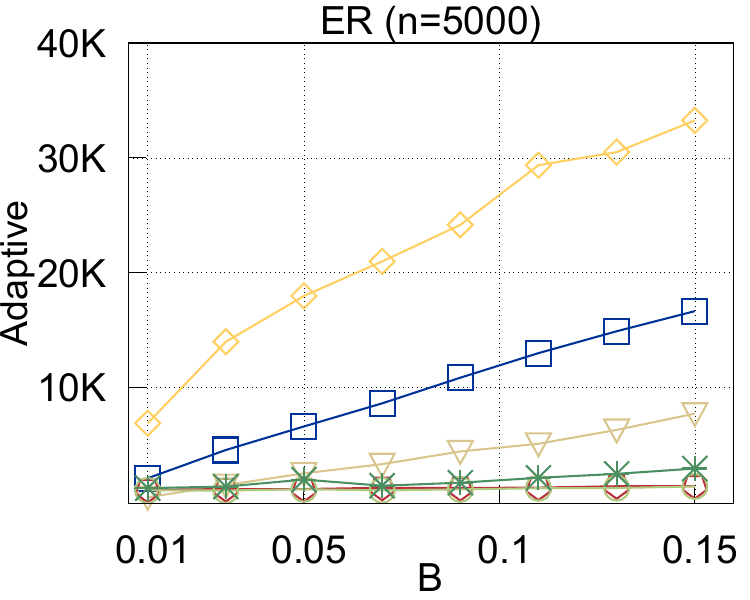}}
		\\
		\footnotesize{	 \hspace{.5cm} (e) \hspace{3.5cm}(f)}
		\caption{Performance of algorithms for $\NSMK$ on three instances: (a), (b) Revenue Maximization; (c), (d) Image Summarization and (e), (f) Maximum Weighted Cut. The budget values represent fractions of the total cost of all elements.}
		\label{fig:res}
		\vspace{-1em}
	\end{figure}
	\section{Conclusions}
	\label{sec:con}
	Motivated by the challenge of the large scale of input data, in this work, we focus on parallel approximation algorithms based on the concept of adaptive complexity. Moreover, the requirement of improving the approximation factor while decreasing the adaptivity down to $\log(n)$ motivates us to propose a competitive new algorithm.	
	We have proposed an efficient parallel algorithm $\AST$ based on a novel alternate threshold greedy strategy.
	To our knowledge, our $\AST$ algorithm is the first to achieve a constant factor approximation of $7+\epsilon$ for the above problem in the aforementioned adaptivity. 
	Our algorithm also expresses the superiority in solution quality and computation complexity compared to state-of-the-art algorithms via some illustrations in the experiment in three real-world applications. In the future, we will address another valuable question: can we reduce the query complexity of parallelized algorithms for the $\SMK$ problem?
	\section*{Acknowledgments}
	The first author (Tan D. Tran) was funded by the Master, PhD Scholarship Programme of Vingroup Innovation Foundation (VINIF), code VINIF.2023.TS.105. This work has been carried
	out partly at the Vietnam Institute for Advanced Study in Mathematics
	(VIASM). The second author (Canh V. Pham) would like to thank
	VIASM for its hospitality and financial support.
	\bibliographystyle{named}
	\bibliography{smk-ijcai} 
	\newpage
	\clearpage
	\appendix
	\onecolumn
		\section*{Appendix}
	\section{$\Rand$ Algorithm}
	In this section, we recap $\Rand$ \cite{Cui-aaai23} (Algorithm 2 of that paper), a frequently used subroutine in the proposed algorithms. For a discussion of the intuition behind $\Rand$
	and rigorous proof of Lemma~\ref{lem:cost-ex}, we refer the reader to \cite{Cui-aaai23} and its full version \cite{Cui-aaai23-full}. 
	\begin{algorithm}[h]
	\SetNlSty{text}{}{:}
	\KwIn{$\theta$, $I$, $M$, $\epsilon$, $f(\cdot)$, $c(\cdot)$} 
	$ A \leftarrow \emptyset, count \leftarrow 0$
	\\
	$ L \leftarrow \{u \in I: \frac{f(u|A)}{c(u)}\geq \rho \wedge c(A\cup \{u\})\leq B\}$;
	\\
	\While{$L\neq \emptyset \wedge count <M$ }
	{
	$\{v_1, v_2, \ldots, v_d\} \leftarrow \Get(A, L, c(\cdot))$;
	\\
	\ForEach{$i \in \{0,1, \ldots, d\}$}
	{
	$V \leftarrow \{v_1, v_2 , \ldots, v_i\}$, $G_i \leftarrow A\cup V_i$;
	\\
	$E^+_i \leftarrow \{u \in L:\frac{f(u|G_i)}{c(u)}\geq \rho c(G_i\cup \{u\})\leq B\}$;
	\\
	$E^-_i \leftarrow \{u \in L:f(u|G_i)<0 \}$;
	\\
	$D_i\leftarrow \{v_j: j\in [i] \wedge f(v_j|A\cup V_{j-1})<0 \}$;
	}
	Find $t_1 \leftarrow \min_{i\leq d}\{c(E^+)\leq (1-\epsilon)c(L)\}$, $t_2 \leftarrow \min_{i\leq d}\{\epsilon \sum_{u \in E^+_i}\}f(u|G_i)\leq \sum_{u \in E^-_{i}}|f(u|G_i)| + \sum_{v_j\in D_i}|f(v_j|A\cup V_{j-1})|$
	\\
	$t^* \leftarrow \min\{t_1, t_2\}$; $U \leftarrow U \{V_{t^*}\}$;
	\\
	With probability $p$ do:
	\\
	\ \ \ \ 	$A\leftarrow A\cup V_{t^*}$
	\\
	\ \ \ \ \textbf{If}	$t_2\leq t_1$ \textbf{then} $count\leftarrow count+1 $
	\\
	$L \leftarrow \{u\in L\setminus U: \frac{f(u|A)}{c(u)}\geq \rho \wedge c(A\cup \{u\})\leq B \}$
	}
	\Return $(A, U, L)$
	\caption{$\Rand$}
	\label{alg:ranbatch}
	\end{algorithm}
	\begin{algorithm}[h]
	\SetNlSty{text}{}{:}
	\KwIn{$\theta$, $I$, $M$, $\epsilon$, $f(\cdot)$, $c(\cdot)$} 
	$ A \leftarrow \emptyset, count \leftarrow 0$
	\\
	\While{$X\neq \emptyset$}
	{
	Draw $a_i$ uniformly at random from $X$
	\\
	$A \leftarrow [a_1, a_2, \ldots, a_i]$
	\\
	$X \leftarrow \{e \in X \setminus {a_i}: c(e)+c(A)+ c(S) \leq B\}$
	\\
	$i \leftarrow i+1$
	}
	\Return $A$
	\caption{$\Get(A, I, c(\cdot))$}
	\label{alg:rb}
	\end{algorithm}
	\subsection{Proof of Lemma~\ref{lem:add-element}}
	The proof of Lemma~\ref{lem:add-element} implies from the proof of Lemma~\ref{lem:cost-ex}.\textbf{We write down the details of the proof for the sake of completeness. }
	We first recap the Lemma 2 in~\cite{Cui-aaai23-full} for supporting the proof of Lemma~\ref{lem:add-element}.
	\begin{lemma}[Lemma 2 in~\cite{Cui-aaai23-full}]
	For any $V_{t^*}$ found in Lines 11-12 of $\Rand$ and any $V_{t^*}$, let $\lambda(u)$
	denote the set of elements in $V_{t^*}$ selected before $u$ (note that $V_{t^*}$ is an ordered list according to Line 6 of $\Rand$), and $\lambda(u)$ does not include $u$. Let $A$ and $U$ be the sets returned
	by $\Rand(\rho, I, M, p, \epsilon, f(\cdot), c(\cdot))$, where the elements in $U$ are $\{u_1, u_2,\ldots, u_s\}$ (listed
	according to the order they are added into $U$). Let $u_j$ be a ``dummy element'' with zero
	marginal gain and zero cost for all $s < j \leq|I|$. Then we have:
	\begin{align}
	\forall j \in [|I|]: \E[f(u_j| \{u_1, \ldots, u_{j-1} \} \cap A\cup \lambda(u_j)) |\F_{j-1}]\geq (1-\epsilon)^2\rho \cdot \E[c(u_j) |\F_{j-1}] 
	\end{align}
	where $\F_{j-1}$ denotes the filtration capturing all the random choices made until the moment right before $u_j$ is selected.
	\label{lem:han-lem2}
	\end{lemma}
	
	Consider the sequence $\{u_1, u_2, \ldots, u_{s+1}, \ldots, u_{|I|}\}$ to be defined in Lemma \ref{lem:han-lem2}. For each $j \in [|I|]$, define the random variable $\delta_1(u_j)=f(u_j|\{u_1, \ldots, u_{j-1}\}\cap A \cup \lambda(u_j))$ if $u_j \in A$ and $\delta_1(u_j)=0$ otherwise, and also define $\delta_2(u_j)=c(u_j)$ if $u_j \in A$ and $\delta_2(u_j)=0$ otherwise. So we have $f(A)\geq \sum_{j=1}^{|I|}$ and $c(A)=\sum_{j=1}^{|I|}\delta_2(u_j)$. Due to the linearity of expectation and the law of total expectation, we only need to prove
	\begin{align}
	\forall j \in [|I|], \forall \F_{j-1}: \E[f(\delta_1(u_j)| \{u_1, \ldots, u_{j-1} \} \cap A\cup \lambda(u_j)) |\F_{j-1}]\geq (1-\epsilon)^2\rho \cdot \E[\delta_2(u_j) |\F_{j-1}] 
	\end{align}
	where $\F_{j-1}$ is the filtration defined in Lemma \ref{lem:han-lem2}. This trivially holds for $j > s$. For any $j\leq s$, we have
	\begin{align}
	\E[\delta_1(u_j)|\F_{j-1}]&= \E[\E[\delta_1(u_j)|\F_{j-1}, u_j]|\F_{j-1}]
	\\
	&= p \E[f(u_j|\{u_1, \ldots, u_{j-1} \cap A \cup \lambda(u_j)\})|\F_{j-1}],
	\end{align}
	where the second equality is because each $u_j \in U$ is added into $A$ with probability $p$, which is independent of the selection of $u_j$. Similarity, we can prove $\E[\delta_2(u_j)|\F_{j-1}]=p\cdot \E[c(u_j)|\F_{j-1}]$. Combine two cases with $p=1$ and Lemma~\ref{lem:han-lem2} we have:
	\begin{align}
	\E[f(u_j| \{u_1, \ldots, u_{j-1} \} \cap A\cup \lambda(u_j)))]=	(1-\epsilon)^2\rho \E[c(u_j)]
	\end{align}
	which implies the proof.
	\section{Missing proofs of Section~\ref{sec:alg}}
	We first provide some notations used for the proofs. 
	\begin{itemize}
	\item $X$, $Y$ are the sets after the the first loop of $\AST$.
	\item Supposing $X$ and $Y$ ordered: $X=\{x_1, x_2, \ldots, x_{|X|}\}$, $Y=\{y_1, y_2, \ldots, y_{|Y|}\}$, we conduct: $X^i=\{x_1, x_2, \ldots, x_i\}$, $Y^i=\{y_1, y_2, \ldots, y_{i}\}$.
	\item 	 $t=\max\{i\in \mathbb{N}: c(X^i)+c(r)\leq B-c(r)\}$, $u=\max\{i\in \mathbb{N}: c(Y^i)+c(r)\leq B-c(r)\}$.
	\item $X_i$ and $Y_i$ are $X$ and $Y$ after iteration $i$, respectively.
	\item For $x \in X\cup Y$, we assume that $x$ is added into $X$ or $Y$ at iteration $l(x)$. 
	\item $\theta^X_{i}$ is $\theta^X$ at iteration $i$, $\theta^Y_{i}$ is $\theta^Y$ at iteration $i$.
	\item $\theta^X_{last}$ and  $\theta^Y_{last}$ are $\theta^X$ and $\theta^Y$ at the last iteration when $X$ and $Y$ are considered to update.
	\item For an element $e \in X\cup Y$, we denote $X_{<e}$, $Y_{<e}$, $\theta^X_e$ and $\theta^Y_e$ as $X$, $Y$, $\theta^X$ and $\theta^Y$ right before $e$ is selected into $X$ or $Y$, respectively. 
	\item $O$ is an optimal solution , $O_1$ is an optimal solution of $\SMK$ problem over the instance $(V_1, f, B)$.
	\item $O'=O_1 \setminus X_1$, $O^r= O'\setminus \{r\}$.
	\end{itemize}
	\subsection{Proof of Lemma~\ref{lem:r-big}}
	\textbf{Prove  a), b)}.  	In this case we have $c(X^t)\leq B-c(r)\leq \epsilon B$, we consider following cases:
	\\
	\textbf{Case 1.} If $X^t=X$, denote $Y_{(t)}=Y_{<x_t}$, i.e., the set $Y$ right before the algorithm obtains $X^t$. For any element $e\in O'\setminus (X\cup Y_{(t)})$, we have $c(X^t)+c(e)\leq B$. Therefore we consider $O'\setminus (X\cup Y_{(t)})=O'^1\cup O'^2$, where $O'^1=\{e\in O'\setminus (X\cup Y_{(t)}): f(e|X)< c(e) \theta_{last}^X\}$, $O'^2=\{e\in O'\setminus (X\cup Y_{(t)}): f(e|X)\geq c(e) \theta_{last}^X\}$. By setting $\Delta$, at any iteration $i$ we have
	\begin{align}
	\frac{8\alpha f(S_0)(1-\epsilon)}{(1-8\delta)\epsilon B}=\Gamma(1-\epsilon) \geq \Gamma(1-\epsilon)^i \geq \Gamma(1-\epsilon)^\Delta \geq \frac{\epsilon f(S_0)(1-\epsilon)}{B}
	\end{align}
	and therefore $\theta^X_{last}\leq \frac{\Gamma(1-\epsilon)^\Delta}{1-\epsilon}\leq \frac{\epsilon\opt}{B}$. 
	By applying Lemma~\ref{lem:bound-T} and Lemma~\ref{lem:cost-ex}
	we have
	\begin{align}
	f(X\cup O')-f(X)&\leq \sum_{e\in O'\setminus X}f(e|X) \ \ \ \ \mbox{(By the submodularity of $f$)}
	\\
	&=\sum_{e\in O'\cap Y_{(t)}}f(e|X)+\sum_{e\in O'\setminus (X\cup Y_{(t)})}f(e|X)
	\\
	&=\sum_{e\in O'\cap Y_{(t)}}f(e|X)+\sum_{e\in O'^1}f(e|X)+ \sum_{e\in O'^2}f(e|X)
	\\
	&\leq \sum_{e\in O'\cap Y_{(t)}}\frac{\E[f(e|Y_{<e})]}{(1-\epsilon)^3}+\epsilon\opt+ c(O'^1)\theta^X_{last} + \frac{\opt}{\epsilon M} \ \ \ \mbox{(By applying Lemma \ref{lem:cost-ex} and Lemma \ref{lem:bound-T})}
	\\
	&\leq \frac{\E[f(Y_{(t)})]}{(1-\epsilon)^3}+\epsilon\opt+ B\theta_{last}^X + \frac{\opt}{\epsilon M}
	\\
	&< \frac{\E[f(Y)]}{(1-\epsilon)^3}+\epsilon\opt+\epsilon \opt + \frac{\opt}{\epsilon M}.
	\end{align}
	It follows that
	\begin{align}
	f(X\cup O') & \leq f(X)+\frac{\E[f(Y)]}{(1-\epsilon)^3}+ (2\epsilon+\frac{1}{\epsilon M})\opt
	\\
	& < (1+\frac{1}{(1-\epsilon)^3})\E[f(S)]+(2\epsilon+\frac{1}{\epsilon M})\opt.
	\end{align}
	\textbf{Case 2.} If $X^t\neq X$, we now consider following cases.
	Consider the iteration $z, z\geq 1$ that 
	\begin{align}
	\frac{(1-\epsilon)\alpha\opt}{ B}\leq \theta_z=\frac{\alpha\Gamma (1-\epsilon)^z}{ \epsilon B} \leq \frac{\alpha\opt}{ B}.
	\end{align}
	We define an odd integer $l$ as follows
	\begin{align}
	l=\begin{cases}
	z, \mbox{if $z$ is odd}
	\\
	z+1, \mbox{otherwise}.
	\end{cases}
	\end{align}
	It follows that $$\frac{(1-\epsilon)^2\alpha\opt}{B}\leq \theta_z\leq \theta^X_l \leq \theta_z\leq \frac{\alpha\opt}{B}$$
	We consider two following cases:
	\\
	\textbf{Case 2.1.} If $X^{t+1} \subseteq X_l$, we further consider two sub-cases
	\begin{itemize}
	\item	 If $c(X_l)\geq (1-\epsilon)B$, by Lemma \ref{lem:cost-ex} we have
	$$\E[f(S)]\geq \E[f(X_l)]\geq \E[c(X_l)](1-\epsilon)^2\theta^X_l\geq (1-\epsilon)^5\alpha\opt.$$ 
	\item 	If $c(X_l)< (1-\epsilon)B$. Since $c(O^r)\leq \epsilon B$, $c(X_l)+c(e)< B$, for all $e\in O^r$. Combine this with the Lemma~\ref{lem:bound-T} and Lemma~\ref{lem:cost-ex},
	we have
	\begin{align}
	f(X_l\cup O')-f(X_l \cup \{r\})&\leq \sum_{e \in O'^r\setminus X_l}f(e|X_l)
	\\
	& \leq  \sum_{e \in O'^r\cap Y_{l-1}}f(e|X_l)+ \sum_{e \in O'^r\setminus (X_l\cup Y_{l-1})}f(e|X_l)
	\\
	&< \sum_{e \in O'^r\cap Y_{l-1}}\frac{\E[f(e|Y_{<e})]}{(1-\epsilon)^3}+\epsilon\opt+ c(O'^r)\theta^X_l + \frac{\opt}{\epsilon M}
	\\
	&\leq \frac{\E[f(Y_{l-1})]}{(1-\epsilon)^3}+\epsilon\alpha\opt+\alpha\epsilon\opt+ \frac{\opt}{\epsilon M}
	\\
	&\leq  \frac{\E[f(Y)]}{(1-\epsilon)^3}+2\alpha\epsilon\opt + \frac{\opt}{\epsilon M}.
	\end{align}
	Combine this with the selection rule of final solution: $f(X_l\cup \{r\})\leq f(S)$, we have
	\begin{align}
	f(X_l\cup O')& <  f(X_l\cup\{r\} )+\frac{\E[f(Y)]}{(1-\epsilon)^3}+2\alpha\epsilon\opt + \frac{\opt}{\epsilon M}
	\\
	&\leq (1+\frac{1}{(1-
	\epsilon)^3})\E[f(S)]+2\alpha\epsilon\opt + \frac{\opt}{\epsilon M}.
	\end{align}
	\end{itemize}
	\textbf{Case 2.2.} If $ X_l\subset X^{t+1}$, then $c(X_l)+c(e)\leq B$ for all $e\in O_1$. We thus also have
	\begin{align}
	f(X_l\cup O')-f(X_l \cup \{r\})&\leq \sum_{e \in O'^r\setminus X_l}f(e|X_l)
	\\
	& = \sum_{e \in O'^r\cap Y_{l-1}}f(e|X_l)+ \sum_{e \in O'^r\setminus (X_l\cup Y_{l-1})}f(e|X_l)
	\\
	&< \sum_{e \in O'^r\cap X_{l-1}}\frac{\E[f(e|Y_{<e})]}{(1-\epsilon)^3}+\epsilon\opt+ c(O'^r)\theta^X_l + \frac{\opt}{\epsilon M}
	\\
	&\leq \frac{\E[f(Y_{l-1})]}{(1-\epsilon)^3}+\epsilon\alpha\opt+\alpha\epsilon\opt+ \frac{\opt}{\epsilon M}
	\\
	&\leq  \frac{\E[f(Y)]}{(1-\epsilon)^3}+2\alpha\epsilon\opt + \frac{\opt}{\epsilon M}.
	\end{align}
	Note that $f(X_l\cup \{r\})\leq f(S)$ due to the selection rule of the second loop. Therefore 
	\begin{align}
	f(X_l\cup O')&\leq (1+\frac{1}{(1-
	\epsilon)^3})\E[f(S)]+2\alpha\epsilon\opt + \frac{\opt}{\epsilon M}.
	\end{align}
	Combine all above cases with note that $\alpha<1$, we obtain the proof.
	\\
	\textbf{Prove  c), d)}. We prove this case by similar argument as previous case and applying Lemma~\ref{lem:bound-T} with $i$ is odd.
	\\
	\textbf{Case 1.} If $Y^u=Y$, denote $X_{(u)}=Y_{<y_u}$, i.e., the set $X$ right before the algorithm obtains $Y^u$. For any element $e\in O'\setminus (Y\cup X_{(u)})$, we have $c(Y^u)+c(e)\leq B$. Therefore we consider $O'\setminus (Y\cup X_{(u)})=O'_1\cup O'_2$, where $O_1'=\{e\in O'\setminus (Y\cup X_{(u)}): f(e|Y)< c(e) \theta_{last}^Y\}$, $O_2'=\{e\in O'\setminus (Y\cup X_{(u)}): f(e|Y)\geq c(e) \theta_{last}^Y\}$. We have:
	$\theta^Y_{last}\leq \frac{\Gamma(1-\epsilon)^\Delta}{1-\epsilon}\leq \frac{\epsilon\opt}{B}$. 
	By applying Lemma~\ref{lem:bound-T} and Lemma~\ref{lem:cost-ex}
	we have
	\begin{align}
	f(Y\cup O')-f(Y)&\leq \sum_{e\in O'\setminus Y}f(e|Y) \ \ \ \ \mbox{(By the submodularity of $f$)}
	\\
	&=\sum_{e\in O'\cap X_{(u)}}f(e|Y)+\sum_{e\in O'\setminus (Y\cup X_{(u)})}f(e|Y)
	\\
	&=\sum_{e\in O'\cap X_{(u)}}f(e|Y)+\sum_{e\in O'_1}f(e|Y)+ \sum_{e\in O'_2}f(e|Y)
	\\
	&< \sum_{e\in O'\cap X_{(u)}}\frac{\E[f(e|X_{<e})]}{(1-\epsilon)^3}+\epsilon\opt+ c(O'_1)\theta^X_{last} + \frac{\opt}{\epsilon M} 
	\\
	&\leq \frac{\E[f(e|X_{<e})]}{(1-\epsilon)^3}+\epsilon\opt+ B\theta_{last}^X + \frac{\opt}{\epsilon M}
	\\
	&\leq  \frac{\E[f(X)]}{(1-\epsilon)^3}+\epsilon\opt+\epsilon \opt + \frac{\opt}{\epsilon M}.
	\end{align}
	It follows that
	\begin{align}
	f(Y\cup O') & < f(Y)+\frac{\E[f(X)]}{(1-\epsilon)^3}+ (2\epsilon+\frac{1}{\epsilon M})\opt
	\\
	& \leq  (1+\frac{1}{(1-\epsilon)^3})\E[f(S)]+(2\epsilon+\frac{1}{\epsilon M})\opt.
	\end{align}
	\textbf{Case 2.} If $Y^u\neq Y$, Consider the iteration $z, z\geq 1$ that 
	\begin{align}
	\frac{(1-\epsilon)\alpha\opt}{ B}\leq \theta_z=\frac{\alpha\Gamma (1-\epsilon)^z}{ \epsilon B} \leq \frac{\alpha\opt}{ B}.
	\end{align}
	We define an even integer $l$ as follows
	\begin{align}
	l=\begin{cases}
	z, \mbox{if $z$ is even}
	\\
	z+1, \mbox{otherwise}.
	\end{cases}
	\end{align}
	It follows that $$\frac{(1-\epsilon)^2\alpha\opt}{B}\leq \theta_z\leq \theta^X_l \leq \theta_z\leq \frac{\alpha\opt}{B}$$
	We consider two following cases:
	\\
	\textbf{Case 2.1.} If $Y^{u+1} \subseteq Y_l$, we further consider two sub-cases:
	\begin{itemize}
	\item	 If $c(Y_l)\geq (1-\epsilon)B$, by Lemma \ref{lem:cost-ex} we have
	$$\E[f(S)]\geq \E[f(Y_l)]\geq \E[c(Y_l)](1-\epsilon)^2\theta^Y_l\geq (1-\epsilon)^5\alpha\opt.$$ 
	\item 	If $c(Y_l)< (1-\epsilon)B$. Since $c(O'^r)\leq \epsilon B$, $c(Y_l)+c(e)< B$, for all $e\in O'^r$. Combine this with the Lemma~\ref{lem:bound-T} and Lemma~\ref{lem:cost-ex},
	we have
	\begin{align}
	f(Y_l\cup O')-f(Y_l\cup \{r\})&\leq \sum_{e \in O'^r\setminus Y_l}f(e|Y_l)
	\\
	& = \sum_{e \in O'^r\cap X_{l-1}}f(e|Y_l)+ \sum_{e \in O'^r\setminus (Y_l\cup X_{l-1})}f(e|Y_l)
	\\
	&< \sum_{e \in O'^r\cap X_{l-1}}\frac{\E[f(e|X_{<e})]}{(1-\epsilon)^3}+\epsilon\opt+ c(O'^r)\theta^Y_l + \frac{\opt}{\epsilon M}
	\\
	&\leq \frac{\E[f(X_{l-1})]}{(1-\epsilon)^3}+\epsilon\alpha\opt+\alpha\epsilon\opt+ \frac{\opt}{\epsilon M}
	\\
	&\leq  \frac{\E[f(X)]}{(1-\epsilon)^3}+2\alpha\epsilon\opt + \frac{\opt}{\epsilon M}
	\end{align}
	Combine this with the selection rule of final solution, we have
	\begin{align}
	f(Y_l\cup O')& <  f(Y_l)+\frac{\E[f(X)]}{(1-\epsilon)^3}+2\alpha\epsilon\opt + \frac{\opt}{\epsilon M}
	\\
	&\leq (1+\frac{1}{(1-
	\epsilon)^3})\E[f(S)]+2\alpha\epsilon\opt + \frac{\opt}{\epsilon M}
	\end{align}
	\end{itemize}
	\textbf{Case 2.2.} If $ Y_l\subset Y^{u+1}$, then $c(Y_l)+c(e)\leq B$ for all $e\in O'$. We thus also have
	\begin{align}
	f(Y_l\cup O')-f(Y_l \cup \{r\})&\leq \sum_{e \in O'^r\setminus Y_l}f(e|Y_l)
	\\
	& = \sum_{e \in O'^r\cap X_{l-1}}f(e|Y_l)+ \sum_{e \in O'^r\setminus (Y_l\cup X_{l-1})}f(e|Y_l)
	\\
	&< \sum_{e \in O'^r\cap X_{l-1}}\frac{\E[f(e|X_{<e})]}{(1-\epsilon)^3}+\epsilon\opt+ c(O^r)\theta^X_l + \frac{\opt}{\epsilon M}
	\\
	&\leq \frac{\E[f(X_{l-1})]}{(1-\epsilon)^3}+\epsilon\alpha\opt+\alpha\epsilon\opt+ \frac{\opt}{\epsilon M}
	\\
	&\leq  \frac{\E[f(X)]}{(1-\epsilon)^3}+2\alpha\epsilon\opt + \frac{\opt}{\epsilon M}.
	\end{align}
	Note that $f(Y_l\cup \{r\})\leq f(S)$ due to the selection rule of the second loop. Therefore 
	\begin{align}
	f(Y_l\cup O')&< (1+\frac{1}{(1-
	\epsilon)^3})\E[f(S)]+2\alpha\epsilon\opt + \frac{\opt}{\epsilon M}.
	\end{align}
	Combine all cases, we obtain the proof.
	\subsection{Proof of Lemma~\ref{lem:XY-small}}
	We prove the Lemma by carefully considering the following cases:
	\\
	\textbf{Case 1.} If $c(r)< (1-\epsilon)B$, then $B-c(r)> \epsilon B$.
	\\
	\textbf{Case 1.1.} If $c(X)$ and $c(Y)$ are both greater than $B-c(r)$. Since $c(X^{t+1})>\epsilon B$ and $c(Y^{u+1})>\epsilon B$,
	the algorithm must obtain $X^{t+1}$ and $Y^{u+1}$ after $i$-th iteration, $i\geq 3$. We prove this case when \textbf{the algorithm obtains $X^t$ before $Y^u$}. When considering the subset $T\subseteq O'$ in Lemma~\ref{lem:bound-T}, the roles of $X$ and $Y$ are the same. Thus
	we can prove similarly for the remaining case.
	Since the elements in $X$ are added at the odd iterations, any element $e\in O'\setminus X^t$ was not added into $X$ at iteration $s=l(x_{t+1})-2$. 
	It is easy to find that $X_{s}\subseteq X^t$. By applying Lemma~\ref{lem:bound-T}, we have
	\begin{align}
	f(O'\cup X^{t})-f(X^{t}\cup \{r\})&\leq \sum_{e\in O'^r\setminus X^t}f(e|X^t)
	\\
	&= \sum_{e\in Y_{<x_t}\cap O'^r}f(e|X^{t})+ \sum_{e\in O'^r\setminus (X^t\cup Y_{<x_t})}f(e|X^{t})
	\\
	& < \sum_{e\in Y_{<x_t}\cap O'^r}\frac{\E[f(e|Y_{<e})]}{(1-\epsilon)^3}+\epsilon \opt+ \sum_{e\in O'^r\setminus (X^t\cup Y_{<x_t})}f(e|X^{t}).\label{ine:com1}
	\end{align}
	Besides, applying Lemma~\ref{lem:bound-T} again gives:
	\begin{align}
	f(O'\cup Y^u)-f(Y^u\cup \{r\})
	&=\sum_{e \in O'^r \cap X_{<y_u}}f(e|Y^u)+\sum_{e \in O'^r\setminus ( X_{<y_u}\cup Y^u)}f(e|Y^u)
	\\
	&\leq \sum_{e \in O'^r \cap X_{<y_u}}\frac{\E[f(e|X_{<e})]}{(1-\epsilon)^3}+\epsilon \opt+\sum_{e \in O'^r\setminus ( X_{<y_u}\cup Y^u)}f(e|Y^u).\label{ine:com2}
	\end{align}
	By the submodularity of $f$ and $X^t\cap Y^u=\emptyset$, we have
	\begin{align}
	f(O')&\leq f(O'\cup X^t)+f(O'\cup Y^u).
	\end{align}
	From now on, it is necessary to bound some terms of the right hand of~\eqref{ine:com1} and \eqref{ine:com2}.
	\begin{itemize}
	\item[a)] Bound of 
	$
	\sum_{e\in O'^r\setminus (X^{t}\cup Y_{<x_t})}f(e|X^{t})+ \sum_{e \in O'^r \cap X_{<y_u}}\frac{\E[f(e|X_{<e})]}{(1-\epsilon)^3}
	$.
	\\
	- \underline{If $x_{t+1} \notin O'^r$}, 
	$O'^r \setminus (X^t\cup Y_{<x_t})=O'^r \setminus (X^{t+1}\cup Y_{<x_t})$ and $O'^r\cap X^t=O'^r\cap X^{t+1}$.
	Any element $e\in O'^r\setminus (X^{t+1}\cup Y_{<x_t})$ was not added in to $X$ at iteration $s=l(x_{t+1})-2$. By applying Lemma~\ref{lem:cost-ex}, we have
	\begin{align}
	\sum_{e\in O'^r\setminus (X^{t}\cup Y_{<x_t})}f(e|X^{t})&=\sum_{e\in O'^r\setminus (X^{t+1}\cup Y_{<x_t})}f(e|X^{t})
	\\
	&\leq \sum_{e\in O'^r\setminus (X^{t+1}\cup Y_{<x_t})}f(e|X_{s})
	\\
	&< c(O'^r\setminus (X^{t+1}\cup Y_{<x_t}))\frac{\theta^X_{(t+1)}}{(1-\epsilon)^2}+\frac{\opt}{\epsilon M}. \label{com-last}
	\end{align}
	On the other hand, since $c(X^{t+1})>B-c(r)\geq c(O'^r)$ so $c(X^{t+1}\setminus O'^r)> c(O'^r\setminus X^{t+1})\geq c(O'^r\setminus (X^{t+1}\cup Y_{<x_t}))$. Combine this with \eqref{com-last}, we have
	\begin{align}
	\sum_{e \in X^{t+1}\setminus O'^r}f(e|X_{<e})&\geq c(X^{t+1}\setminus O'^r)\theta^X_{(t+1)}
	\\
	&\geq(1-\epsilon)^2\frac{c(X^{t+1}\setminus O'^r)}{ c(O'^r\setminus (X^{t}\cup Y_{<x_t}))} \left( \sum_{e\in O'^r\setminus (X^{t}\cup Y_{<x_t})}f(e|X^{t})-\frac{\opt}{\epsilon M} \right) 
	\\
	&> (1-\epsilon)^2 \left( \sum_{e\in O'^r\setminus (X^{t}\cup Y_{<x_t})}f(e|X^{t})-\frac{\opt}{\epsilon M} \right). \label{com-2}
	\end{align}
	From~\eqref{com-2} and~\eqref{com-last} with a note that $X_{<y_u}\subseteq X^t$ we obtain 
	\begin{align}
	&\sum_{e\in O'^r \setminus (X^{t}\cup Y_{<x_t})} f(e|X^t)+\sum_{e\in O'^r \cap X_{<y_u}} \frac{\E[f(e|X_{<e})]}{(1-\epsilon)^3}
	\\
	& \leq \frac{\sum_{e \in X^{t+1}\setminus O'^r}f(e|X_{<e})}{(1-\epsilon)^2}+\sum_{e\in O'^r \cap X^{t}} \frac{\E[f(e|X_{<e})]}{(1-\epsilon)^3}+\frac{\opt}{\epsilon M}
	\\
	& \leq \frac{\sum_{e \in X^{t+1}\setminus O'^r}f(e|X_{<e})}{(1-\epsilon)^2}+\sum_{e\in O'^r \cap X^{t+1}} \frac{\E[f(e|X_{<e})]}{(1-\epsilon)^3}+\frac{\opt}{\epsilon M}
	\\
	&	<\E\left( \frac{\sum_{e \in X^{t+1}\setminus O^r}f(e|X_{<e})}{(1-\epsilon)^3}+\sum_{e\in O^r \cap X^{t+1}} \frac{f(e|X_{<e})}{(1-\epsilon)^3}\right) +\frac{\opt}{\epsilon M}
	\\
	&	< \frac{\E[f(X^{t+1})]}{(1-\epsilon)^3}+\frac{\opt}{\epsilon M} \ \ \mbox{(Due to $X^{t+1}=(X^{t+1}\setminus O^r) \cup (O^r\cap X^{t+1})$)}
	\\
	& \leq \frac{\E[f(S)]}{(1-\epsilon)^3}+\frac{\opt}{\epsilon M}.
	\end{align}
	\underline{- If $x_{t+1} \in O'^r$}, $O'^r\cap X^{t+1}=(O'^r\cap X^t)\cup \{x_{t+1}\}$ and $O'^r \setminus (X^t\cup Y^{<x_t})=O'^r \setminus (X^{t+1}\cup Y^{<x_t}) \cup \{x_{t+1}\}$. 
	Therefore
	\begin{align}
	&\sum_{e\in O'^r \setminus (X^{t}\cup Y_{<x_t})} f(e|X^t)+\sum_{e\in O'^r \cap X_{<y_u}} \frac{\E[f(e|X_{<e})]}{(1-\epsilon)^3}
	\\
	&\leq \sum_{e\in O'^r \setminus (X^{t+1}\cup Y_{<x_t})} f(e|X^t)+f(x_{t+1}|X^t)+\sum_{e\in O'^r \cap X^{t}} \frac{\E[f(e|X_{<e})]}{(1-\epsilon)^3}
	\\
	& \leq \sum_{e\in O'^r \setminus (X^{t+1}\cup Y_{<x_t})} f(e|X^t)+\sum_{e\in O'^r \cap X^{t+1}} \frac{\E[f(e|X_{<e})]}{(1-\epsilon)^3} \label{ine:tra1}
	\\
	&< \frac{\sum_{e \in X^{t+1}\setminus O'^r}f(e|X_{<e})}{(1-\epsilon)^2}+\sum_{e\in O'^r \cap X^{t+1}} \frac{\E[f(e|X^{<e})]}{(1-\epsilon)^3}+\frac{\opt}{\epsilon M} \label{ine:tra2}
	\\
	&< \frac{\E[f(X^{t+1})]}{(1-\epsilon)^3}+\frac{\opt}{\epsilon M}
	\\
	&\leq \frac{\E[f(S)]}{(1-\epsilon)^3}+\frac{\opt}{\epsilon M},
	\end{align} 
	where the transform from \eqref{ine:tra1} to \eqref{ine:tra2} is due to applying \eqref{com-2}.
	\item[b)] Bound of  $\sum_{e \in O'^r\setminus ( X_{<y_u}\cup Y^u)}f(e|Y^u)$. Any element $e \in O'^r\setminus Y^u$ was not added into $Y^u$ at the iteration $l(y_u)-2$. By applying Lemma~\ref{lem:cost-ex} with a note that $c(O'^r\setminus ( X_{<y_u}\cup Y^u))\leq B-r\leq c(Y^{u+1})$, we obtain
	\begin{align}
	\sum_{e \in O'^r\setminus ( X_{<y_u}\cup Y^u)}f(e|Y^u)&\leq c(O'^r\setminus ( X_{<y_u}\cup Y^u))\frac{\theta^Y_{(u+1)}}{(1-\epsilon)^2}+\frac{\opt}{\epsilon M}
	\\
	& \leq c(Y^{u+1})\frac{\theta^Y_{(u+1)}}{(1-\epsilon)^2}+\frac{\opt}{\epsilon M}
	\\
	& \leq \frac{f(Y^{u+1})}{(1-\epsilon)^4}+\frac{\opt}{\epsilon M}.
	\end{align}
	Put the above bounds (\textbf{a} and \textbf{b}) into \eqref{ine:com1}, \eqref{ine:com2}, we have
	\begin{align}
	f(O')&\leq f(O'\cup X^t)+f(O'\cup Y^u)
	\\
	&\leq f(X^t\cup \{r\})+f(Y^u \cup \{r\})+\frac{\E[f(S)]}{(1-\epsilon)^3}+ \sum_{e\in Y_{<x_t}\cap O'^r}\frac{f(e|Y_{<e})}{1-\epsilon}+ \frac{f(Y^{u+1})}{(1-\epsilon)^4}+ \frac{4\opt}{\epsilon M}
	\\
	& \leq 2f(S)+\frac{\E[f(S)]}{(1-\epsilon)^3}+ \frac{f(S)}{1-\epsilon}+ \frac{f(S)}{(1-\epsilon)^4}+ \frac{4\opt}{\epsilon M}
	\\
	&< \left( 2+\frac{3}{(1-\epsilon)^4}\right) \E[f(S)]+ \frac{4\opt}{\epsilon M}< \frac{5\E[f(S)]}{(1-\epsilon)^4}+ \epsilon\opt
	\end{align}
	where the last inequality due to $\epsilon M >4/\epsilon$.
	\end{itemize}
	\textbf{Case 2:} If $c(r)\geq (1-\epsilon)B$. Based on Lemma~\ref{lem:r-big}, we consider sub-cases. If one of \textbf{a} or \textbf{c} in Lemma~\ref{lem:r-big} happens, then Lemma~\ref{lem:r-big} holds. If both \textbf{b} and \textbf{d} happen, we have
	\begin{align}
	f(O')&\leq f(O'\cup X')+ f(O'\cup Y') \leq (2+\frac{2}{(1-
	\epsilon)^3})\E[f(S)]+2(2\epsilon+\frac{1}{\epsilon M})\opt
	\end{align}
	which completes the proof.

\end{document}